\documentclass{article}

\usepackage[preprint, nonatbib]{neurips_2019}

\usepackage{amsthm,amsfonts,amsmath,amssymb}
\usepackage{mathrsfs}
\usepackage{bbm}
\usepackage{thm-restate}
\usepackage{mathtools}
\usepackage{verbatim}
\usepackage{comment}

\usepackage{microtype}
\usepackage{graphicx}
\usepackage{subfigure}
\usepackage{booktabs}

\usepackage[linesnumbered, ruled,vlined]{algorithm2e}
\usepackage{algorithm,algorithmic}

\newtheorem{lemma}{Lemma}
\newtheorem{theorem}{Theorem}

\newtheorem{definition}{Definition}

\usepackage{xcolor}

\title{Provably Robust Blackbox Optimization\\ for Reinforcement Learning}

\author{%
  Krzysztof Choromanski\thanks{Equal contribution.}\\
  Google Brain Robotics\\
   \texttt{kchoro@google.com} \\
   \And
   Aldo Pacchiano\footnotemark[1]\\
  UC Berkeley\\
   \texttt{pacchiano@berkeley.edu} \\
   \And
  Jack Parker-Holder\footnotemark[1]\\
  Columbia University\\
   \texttt{jh3764@columbia.edu} \\
   \And
  Yunhao Tang\\
  Columbia University\\
   \texttt{yt2541@columbia.edu} \\
   \And
   Deepali Jain \\
   Google Brain Robotics\\
   \texttt{jaindeepali@google.com}\\
   \And
   Yuxiang Yang \\
   Google Brain Robotics\\
   \texttt{yxyang@google.com}\\
   \And
   Atil Iscen \\
   Google Brain Robotics\\
   \texttt{atil@google.com}\\
   \And
  Jasmine Hsu\\
  Google Brain Robotics\\
   \texttt{hellojas@google.com} \\
   \And
   Vikas Sindhwani\\
  Google Brain Robotics\\
   \texttt{sindhwani@google.com} \\
   \\[-50.0ex]
   }

\begin{document}
\maketitle

%===============================================================================

\begin{abstract}
Interest in derivative-free optimization (DFO) and ``evolutionary strategies" (ES) has recently surged in the Reinforcement Learning (RL) community, with growing evidence that they can match state of the art methods for policy optimization problems in Robotics.  However, it is well known that DFO methods suffer from prohibitively high sampling complexity. They can also be very sensitive to noisy rewards and stochastic dynamics. In this paper, we  propose a new class of algorithms, called Robust Blackbox Optimization (RBO).
Remarkably, even if up to $23\%$ of all the measurements are arbitrarily corrupted, RBO can provably recover gradients to high accuracy. RBO relies on learning gradient flows using robust regression methods to enable off-policy updates. On several  $\mathrm{MuJoCo}$ robot control tasks, when all other RL approaches collapse in the presence of adversarial noise, RBO is able to train policies effectively. We also show that RBO can be applied to legged locomotion tasks including path tracking for quadruped robots.
\end{abstract}

% Several attempts were made to modify ES algorithms in order to reduce % their sampling complexity, (such as improving their exploration via % carefully chosen ensambles of perturbation directions \cite{}), providing % only very moderate gains so far.

% Two or three meaningful keywords should be added here
%\keywords{Derivative-free Optimization, Reinforcement Learning, Evolutionary Strategies, Policy Search} 

%===============================================================================

\section{Introduction}

%Interest in Blackbox (aka Derivative-free) Optimization [DFO], \cite{conn2009introduction}) \& Evolutionary strategies (ES, \cite{ES}) has recently surged, with growing evidence that they may be a competitive alternative to state of the art policy gradients methods \cite{schulman2017proximal}, \cite{schulman2015trust}, \cite{lillicrap2015continuous} in Reinforcement Learning tasks.  They are simpler to implement, do not rely on the specific internal structure of the MDP problem (thus can be used in general blackbox optimization) and often outperform the aforementioned.

It is appealing to reduce policy learning tasks arising in Robotics to instances of  blackbox optimization problems of the form, 
\begin{equation}
\max_{\theta \in \mathbb{R}^{d}} F(\theta).
\label{eq:blackbox}
\end{equation}
Above, $\theta$ encodes a policy $\pi_\theta:\mathcal{S} \rightarrow \mathcal{A}$, where $\mathcal{S}$ and $\mathcal{A}$ denote the state and action spaces, and the function $F$ maps $\theta$ to the total expected reward when the robot applies  $\pi_\theta$ recursively in a given environment. In this context, the ``blackbox" is an opaque physics simulator or even the robot hardware interacting with a real environment with unknown dynamics. As a consequence, the function $F$ only admits point evaluations with no explicit analytical gradients available for an optimizer to exploit. 

Blackbox methods and the so-called ``evolutionary strategies" (ES) are instances of derivative-free optimization (DFO) \cite{ES, wierstra14a, horia, stanley, seha} that aim to maximize $F$ by applying various random search~\cite{kolda2003optimization} techniques, while avoiding explicit gradient computation.   Typically, in each epoch the policy parameter vector $\theta$ is updated using a gradient ascent rule that has the following general flavor~\cite{ES,stockholm}:
\begin{equation}
\label{base-equation}
\theta \leftarrow \theta + \eta \widehat{\nabla F}(\theta),~~~\textrm{where}~~~\widehat{\nabla F}(\theta)\approx \frac{1}{l} \sum_{j=1}^{l} w(\theta, \mathbf{g}_{i}) \mathbf{g}_{i},
\end{equation}
and the gradient of $F$ at $\theta$ is estimated by evaluating $F$ at $\theta \pm \mathbf{g}_{i}$ for a certain choice of perturbation directions $\{\mathbf{g}_{1},...,\mathbf{g}_{l}\}$.
%where $\{\mathbf{g}_{1},...,\mathbf{g}_{l}\}$ is a subset of the set of $k$ chosen random perturbations/samples  $\{\mathbf{g}_{1},...,\mathbf{g}_{k}\}$ (for some $l \leq k$) 
%defining perturbed versions $\theta \pm \mathbf{g}_{i}$ of the given policy $\theta$. 
Above, the function $w:\mathbb{R} \rightarrow \mathbb{R}$ translates rewards obtained by the perturbed policies to some weights and $\eta>0$ is a step size. For example, the setting $w(\theta, \mathbf{g}) = \frac{1}{h}[F(\theta + h\mathbf{g}) - F(\theta)]$ where $\mathbf{g}$'s are the canonical directions, corresponds to the ubiquitous finite difference gradient estimator.

Surprisingly, despite not exploiting the internal structure of the RL problem, blackbox methods can be highly competitive with  state of the art policy gradient approaches \cite{schulman2017proximal, schulman2015trust, lillicrap2015continuous, babadi}, 
while admitting much simpler and embarrassingly parallelizable implementations. They can also handle complex, non-differentiable policy parameterizations, non-markovian reward structures and non-smooth hybrid dynamics. Particularly in simulation settings, they remain a serious alternative to classical model-free RL methods despite being among the simplest and oldest policy search techniques~\cite{shu2006numerical,kohl2004policy,peters2006policy}.

On the flip side, blackbox methods are notorious for requiring a prohibitively large number of rollouts.  This is because these  methods are exclusively on-policy and extract a relatively small amount of information from samples, compared to other model-free RL algorithms. The latter make use of the underlying structure (e.g. Markovian property) to derive updates, in particular off-policy methods which maintain and re-use previously collected data \cite{dqn2013}.  Indeed, the ES approach of Salimans et al. \cite{ES} required millions of rollouts on thousands of CPUs to get competitive results. 

Furthermore, several theoretical results expose a fundamental and
unavoidable gap between the performance of optimizers with access to gradients
and those with access to only function evaluations, particularly in the presence of noise~\cite{jamieson2012query}. Even when the blackbox is a convex function, blackbox methods usually need more iterations than the standard gradient methods to converge, at a price that scales with problem dimensionality~\cite{nesterov}. Without considerable care, they are also brittle in the face of noise and can breakdown when rewards are noisy or there is considerable stochasticity in the underlying system dynamics.  

%where different workers compute rewards obtained by different perturbed policies independently. However, despite the wall-clock advantage, ES methods notoriously require a large number of long-horizon rollouts, which quickly becomes a computational bottleneck. For instance, Salimans et al. \cite{ES} require thousands of CPUs to get competitive results. This is because ES methods extract a relatively small amount of information from samples, compared to other model-free RL algorithms which make use of the underlying structure (e.g. Markovian property) to derive updates, in particular off-policy methods which maintain and re-use previously collected data \cite{dqn2013}. 

%Furthermore, policies obtained by updates as in Equation \ref{base-equation}, are very sensitive to noisy measurements, e.g. when the dynamics model used in the simulator does not accurately represent the true dynamics in certain regions of the state space. This is a crucial issue for RL application to real world robotics.

Starting from a natural regularized regression perspective on gradient estimation, we propose two simple enhancements to blackbox/ES techniques. First, we inject off-policy learning by reusing past samples to estimate an entire continuous local gradient field in the neighborhood of the current iterate. Secondly, by drawing on results from compressed sensing and error correcting codes~\cite{dwork2007price,candes2005decoding}, we propose a robust regression LP-decoding framework that is guaranteed to provide provable accurate gradient estimates in the face of up to $23\%$ arbitrary noise, including adversarial corruption, in function evaluations. The resulting method (abbreviated as RBO) shows dramatic resilience to massive measurement corruptions on a suite of $8$ $\mathrm{MuJoCo}$ robot control tasks when competing algorithms appear to fall apart. We also observe favorable comparisons on walking and path tracking tasks on quadruped robots.

We start this paper with a simple regression perspective on blackbox optimization, introduce our algorithm and its off-policy elements with a striking theoretical result on its robustness, followed by a comprehensive empirical analysis on a variety of policy search problems in Robotics.

\section{A Regularized Regression Perspective on Gradient Estimation}

%In order to introduce the regression point of view on ES optimization,
We begin by presenting a natural regression perspective on gradient estimation~\cite{conn2009introduction} for derivative-free optimization. First, recall the notion of a Gaussian smoothing $F_\sigma$ of a given blackbox function $F$ defined as, 
\begin{align}
\begin{split}
F_{\sigma}(\theta) = \mathbb{E}_{\mathbf{g} \in \mathcal{N}(0,\mathbf{I}_{d})}[F(\theta + \sigma \mathbf{g})] = 
(2\pi)^{-\frac{d}{2}}
\int_{\mathbb{R}^{d}}F(\theta + \sigma \mathbf{g})e^{-\frac{\|\mathbf{g}\|^{2}}{2}}d\mathbf{g}.
\end{split}
\end{align}

It turns out that the updates proposed in the Evolutionary Strategies approach of Salimans et. al.\cite{ES} can be written as:
\begin{equation}
\theta \leftarrow \theta + \eta 
\widehat{\nabla}_{\mathrm{MC}}F_{\sigma}(\theta),
\end{equation}
where $\widehat{\nabla}_{\mathrm{MC}}F_{\sigma}(\theta)$
is the Monte Carlo (MC) estimator of the gradient 
$\nabla_{\mathrm{MC}}F_{\sigma}(\theta)$ of $F_{\sigma}$ at $\theta$. 

Since the formula for the gradient $\nabla_{\mathrm{MC}}F_{\sigma}(\theta)$ is itself given as an expectation over Gaussian distribution, namely: 
$\nabla F_{\sigma}(\theta)=\frac{1}{\sigma}\mathbb{E}_{\mathbf{g} \sim \mathcal{N}(0,\mathbf{I}_{d})}[F(\theta + \sigma \mathbf{g})\mathbf{g}]$, MC estimators can be easily constructed, simply by sampling $k$ independent Gaussian perturbations $\sigma \mathbf{g}_{i}$ for $i=1,...,k$ and evaluating $F$ at points determined by these perturbations. There exist several such unbiased MC estimators which apply different variance reduction techniques \cite{stockholm,montreal,choromanski2019unifying,tang2019variance}. Without loss of generality, we take an estimator using \textit{forward finite difference} expressions \cite{stockholm} which is of the form:
\begin{equation}
\widehat{\nabla}^{\mathrm{AT}}_{\mathrm{MC}}F_{\sigma}(\theta) = 
\frac{1}{k\sigma}\sum_{i=1}^{k}(F(\theta + \sigma \mathbf{g}_{i})-F(\theta))\mathbf{g}_{i}.
\end{equation}    

One can notice by analyzing the Taylor expansion of $F$ at $\theta$ that forward finite difference expressions $\frac{F(\theta + \sigma \mathbf{g}_{i})-F(\theta)}{\sigma}$ in the formula above can be reinterpreted as estimations of the dot-products $\nabla F(\theta)^T \mathbf{g}_{i}$.
In other words, by querying blackbox RL function $F$ at $\theta$ with perturbations $\sigma \mathbf{g}_{i}$, one effectively collects lots of noisy estimates of $\nabla F(\theta)^T \mathbf{g}_{i}$.  

The task then is to recover the unknown gradient from these estimates. This observation is the key to formulating blackbox function gradient estimation as a regression problem. This approach has two potential benefits. Firstly, it opens blackbox optimization to the wide class of regularized regression-based methods capable of recovering gradients more accurately than standard MC approaches in the presence of substantial noise. Secondly,  it relaxes the independence condition regarding samples chosen in different iterations of the optimization, allowing for samples from previous iterations to be re-used.

As we will see later, the latter will eventually lead to more sample efficient methods. Interestingly, we will show that the recent orthogonal method for variance reduction in ES\cite{stockholm} can be interpreted as a particular instantiation of the regression-based approach.

%Interestingly, we will show that some of the recent methods for improving ES and based on orthogonally-coupled perturbations leading to variance-reduced Quasi Monte-Carlo estimators \cite{stockholm} can be thought of as particular instantiations of the regression-based approach.

\subsection{A simple regression-based algorithm}

Given scalars $\{F(\theta+\mathbf{z}_{i})\}_{i=1}^{k}$ (corresponding to rewards obtained by different perturbations $\mathbf{z}_i$ of the policy encoded by $\theta$), we formulate the 
regression problem by considering $\{\mathbf{z}_{1},...,\mathbf{z}_{k}\}$ as input vectors with target values $y_i = F(\theta+\mathbf{z}_{i}) - F(\theta)$ for $i=1,...,k$. We propose to produce a gradient estimator by solving the following regression problem:
%handle this regression task by solving the following minimization problem:
\begin{equation}
\widehat{\nabla}_{\mathrm{RBO}}F(\theta) =  
\arg \min_{\mathbf{v} \in \mathbb{R}^{d}} \frac{1}{2k}\|\mathbf{y} - \mathbf{Z}\mathbf{v}\|^{p}_{p} + \alpha \|\mathbf{v}\|^{q}_{q},  
\label{eq:regression}
\end{equation}
where $p, q \geq 1$, $\mathbf{Z} \in \mathbb{R}^{k \times d}$ is a matrix with the $i$th row encoding perturbations $\mathbf{z}_{i}$. The sequences of rows in $\mathbf{Z}$ are 
sampled from some given joint multivariate distribution $\mathbb{P} \in \mathcal{P}(\mathbb{R}^{d} \times ... \mathbb{R}^{d})$ and $\alpha > 0$ is a regularization parameter.

As already mentioned, perturbations $\mathbf{z}_{i}$ do not need to be taken from the Gaussian multivariate distribution and 
they do not even need to be independent.
Note that various known regression methods arise by instantiating the above optimization problems with different values of $p, q$ and $\alpha$.
In particular, $p=q=2$ leads to the ridge regression  \cite{regression}, $p=2$, $q=1$ to Lasso \cite{lasso}, $p=1$, $q=2$ to robust regression with least absolute deviations loss, and and $p=1$, $\alpha=0$ to LP decoding \cite{dwork2007price}. 

\subsection{Using off-policy samples} 
Gradients reconstructed by the above regression problem (Equation \ref{eq:regression}) can be given to the ES optimizer. Furthermore, at any given iteration $t$ the ES optimizer can reuse evaluations of these points $\theta_{t-1} + \sigma \mathbf{g}_{i}^{(t-1)}$ that are closest to current point $\theta_{t}$, .e.g. top $\tau$-percentage for the fixed hyperparameter $0 < \tau < 1$. Thus the regression interpretation enables us to go beyond the rigid framework of independent sets of samples. This algorithm, described in more detail in Algorithm 1 box (l.8 in the algorithm is a simple projection step restricting each parameter vector to be within a domain of allowable policies), plays the role of our base RBO variant and the backbone of our algorithmic approach. It already outperforms state-of-the-art ES methods on benchmark RL tasks. In the next section we will explain how it can be further refined to achieve even better performance.  

\subsection{Regression versus ES with orthogonal MC estimators}

Here we show that ES methods based on orthogonal Monte Carlo estimators \cite{stockholm, montreal}, that were recently demonstrated to improve standard ES algorithms for RL, can be thought of as special cases of the regression approach.

Orthogonal MC estimators rely on pairwise orthogonal perturbations $\sigma \mathbf{g}_{i}$ that can be further renormalized to have length $l=\sigma \sqrt{d}$. The renormalization ensures that the marginal distributions of the orthogonal samples are the same as the unstructured ones, which render the orthogonal MC estimators unbiased. Further, the coupling induces correlation between perturbations for provable variance reduction.

\begin{algorithm}[H]
\textbf{Input:} $F : \Theta \rightarrow \mathbb{R}$, scaling parameter sequence $\{\sigma \}_t$, initial $\theta_0 = \mathbf{u}_0 \in \Theta$, number of perturbations $k$, step size sequence $\{ \eta_t \}_t$, sampling distribution $\mathbb{P} \in \mathcal{P}(\mathbb{R}^{d})$, parameters $p,q,\alpha, \tau$, number of epochs $T$.  \; \\
\textbf{Output:} Vector of parameters $\theta_{T}$. \; \\
1. Initialize $\Theta^{\mathrm{pert}}_{\mathrm{old}} = \emptyset$, $R_{\mathrm{old}} = \emptyset$ ($|\Theta^{\mathrm{pert}}_{\mathrm{old}}|=|R_{\mathrm{old}}|$). \; \\
\For{$t=0, 1, \ldots, T-1$}{
  1. Compute all distances from $\mathbf{u}_{t}$ to $\theta^{\mathrm{pert}}_{\mathrm{old}} \in \Theta^{\mathrm{pert}}_{\mathrm{old}}$. \; \\
  2. Find the closest $\tau$-percentage of vectors from $\Theta^{\mathrm{pert}}_{\mathrm{old}}$ and call this set $\Theta^{\mathrm{near}}_{\tau}$. Call the corresponding subset of $R_{\mathrm{old}}$ as $R^{\mathrm{near}}_{\tau}$. \; \\
  3. Sample $\mathbf{g}^{(t)}_1, \cdots, \mathbf{g}^{(t)}_{k-|\Theta^{\mathrm{near}}_{\tau}|}$ from $\mathbb{P}$. \;\\
  4. Compute $F(\theta_t)$ and $F(\theta_t + \sigma_t \mathbf{g}^{(t)}_j)$ for all $j$. \; \\
  5. Let $\mathbf{Z}_t \in \mathbb{R}^{k \times d}$ be a matrix 
  obtained by concatenating rows given by $\sigma_t \times \mathbf{g}^{(t)}_i$ and those 
  of the form: $\mathbf{p}_{i}-\theta_{t}$, where $\mathbf{p}_{i} \in \Theta^{\mathrm{near}}_{\tau}$.\; \\
  6. Let $\mathbf{y}_t \in \mathbb{R}^k$ be the vector obtained by concatenating values $F(\theta_t + \sigma_t \mathbf{g}^{(t)}_j) - F(\theta_t)$ with those
  of the form: $r_{i} - F(\theta_{t})$, where $r_{i} \in R^{\mathrm{near}}_{\tau}$. \; \\
  7. Let $\widehat{\nabla}_{\mathrm{RBO}}F(\theta_t)$ be the resulting vector after solving the following optimization problem:
  \begin{align*}
    \widehat{\nabla}_{\mathrm{RBO}}F(\theta_t) =  
    \arg \min_{\mathbf{v} \in \mathbb{R}^{d}} \frac{1}{2k}\|\mathbf{y}_{t} - \mathbf{Z}_{t}\mathbf{v}\|^{p}_{p} + \alpha \|\mathbf{v}\|^{q}_{q}, 
  \end{align*}\;\\
  8. Take $\mathbf{u}_{t+1} = \theta_t + \eta_t \widehat{\nabla}_{\mathrm{RBO}}F(\theta_t)$ \; \\
  9. Take $\theta_{t+1} = \arg\max_{\theta \in \Theta} \langle \theta, \mathbf{u}_{t+1} \rangle - \frac{1}{2}\| \theta \|_2^2$.\;\\  
  10. Update $\Theta^{\mathrm{pert}}_{\mathrm{old}}$ to be the set
  of the form $\theta_{t} + \mathbf{z}_{i}$, where $\mathbf{z}_{i}$s are rows of $\mathbf{Z}_{t}$
  and $\theta_{t}$, and $R_{\mathrm{old}}$ to be the set of the corresponding values $F(\theta_{t}+\mathbf{z}_{i})$ and $F(\theta_{t})$.
 }
 \caption{Robust Blackbox Optimization Algorithm via Regression}
\label{Alg:tr-robust_gradient_recovery}
\end{algorithm}

Orthogonal MC estimators can be easily constructed via Gram-Schmidt orthogonalization process from the ensembles of unstructured independent samples \cite{kmchoro}. The following is true:

\begin{lemma}
The class of the orthogonal Monte Carlo estimators using renormalization with $k=d$ orthogonal samples
is equivalent to particular sub-classes of $\mathrm{RBO}$ estimators with $p=q=2$.
\end{lemma}

\begin{proof}
The solution to the ridge regression problem for gradient estimation ($p=q=2$) 
%\begin{equation}       
%\arg \min_{\mathbf{v} \in \mathbb{R}^{d}} \frac{1}{2d}\|\mathbf{y}_{t} %- \mathbf{Z}_{t}\mathbf{v}\|^{2}_{2} + \alpha \|\mathbf{v}\|^{2}_{2}
%\end{equation}
is of the form
\begin{equation}
\widehat{\nabla}_{\mathrm{RBO}}F_{\mathrm{ridge}}(\theta) = 
(\mathbf{Z}_{t}^{\top}\mathbf{Z}_{t} + 2d \alpha \mathbf{I}_{d})^{-1}\mathbf{Z}_{t}^{\top}\mathbf{y}_{t}
\end{equation}

By the assumptions of the lemma we get: $\mathbf{Z}_{t}\mathbf{Z}_{t}^{\top} = \sigma^{2}d \mathbf{I}_{d}$, thus $\mathbf{Z}_{t}^{\top} = \sigma^{2}d\mathbf{Z}_{t}^{-1}$, and we obtain:
\begin{equation}
\widehat{\nabla}_{\mathrm{RBO}}F_{\mathrm{ridge}}(\theta) = 
\frac{1}{d \sigma} \mathbf{G}_{\mathrm{ort}}^{\top}\mathbf{y}_{t}
\cdot \frac{\sigma^{2}}{\sigma^{2}+2\alpha},
\end{equation}
where $\mathbf{G}_{\mathrm{ort}}^{\top}$ is a matrix with rows given by orthogonal Gaussian vectors $\mathbf{g}_{i}^{\mathrm{ort}}$.
Thus, if we take $\sigma = \sigma_{\mathrm{MC}}$, where
$\sigma_{\mathrm{MC}}$ stands for the smoothing parameter in the MC estimator and furthermore, $\eta = \eta_{\mathrm{MC}}\frac{\sigma^{2}+2\alpha}{\sigma^{2}}$, where $\eta_{\mathrm{MC}}$ stands for the steps size in the algorithm using that MC estimator, then the RBO estimator is equivalent to the orthogonal MC and the proof is completed.
\end{proof}
\section{Learning Gradient Flows for off-policy sample reuse}

Algorithm 1 reconstructs the gradient of $F$ only at $\theta$. To refine this algorithm, consider reconstructing the gradient of $F$ at an entire continuous neighborhood of $\theta$ instead of $\theta$ alone. The idea is to use values of $F$ computed in the neighborhood $\mathcal{N}(\theta_{t})$ of $\theta_{t}$ to approximate the gradient field $\mathcal{F}_{\mathrm{grad}}$ of $F$ in the entire neighborhood $\mathcal{N}(\theta_{t})$ rather than just at $\theta_{t}$. This method utilizes past function evaluations $F$ to an even bigger extent. In this approach $k$ function values from iteration $t$ of the algorithm are used to reconstruct several gradients in the neighborhood $\mathcal{N}(\theta_{t})$ of $\theta_{t}$ as opposed to the baseline ES algorithm, where each value is used for only one gradient or Algorithm 1, where some values (from the closest $\tau$-percentage of the new point $\theta_{t+1}$) are reused. 

\vspace{-3mm}
\begin{figure}[H]
\begin{minipage}{1.0\textwidth}
	\subfigure[Vanilla ES]{\includegraphics[width=0.3\textwidth]{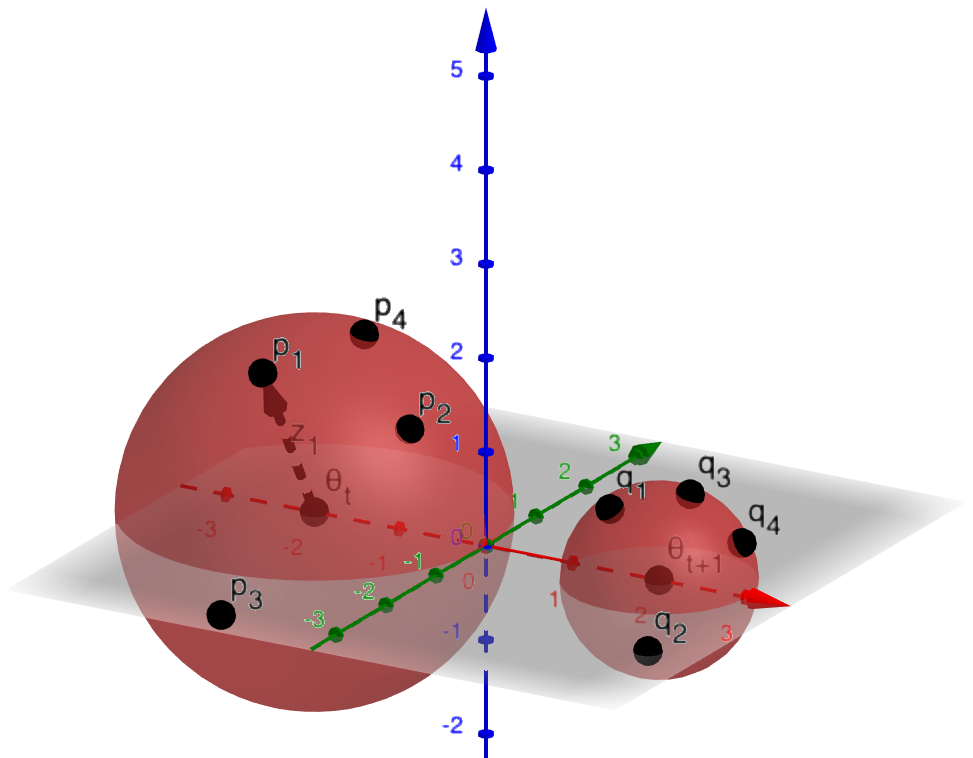}}
	\hspace{0.2cm}
	\subfigure[RBO]{\includegraphics[width=0.3\textwidth]{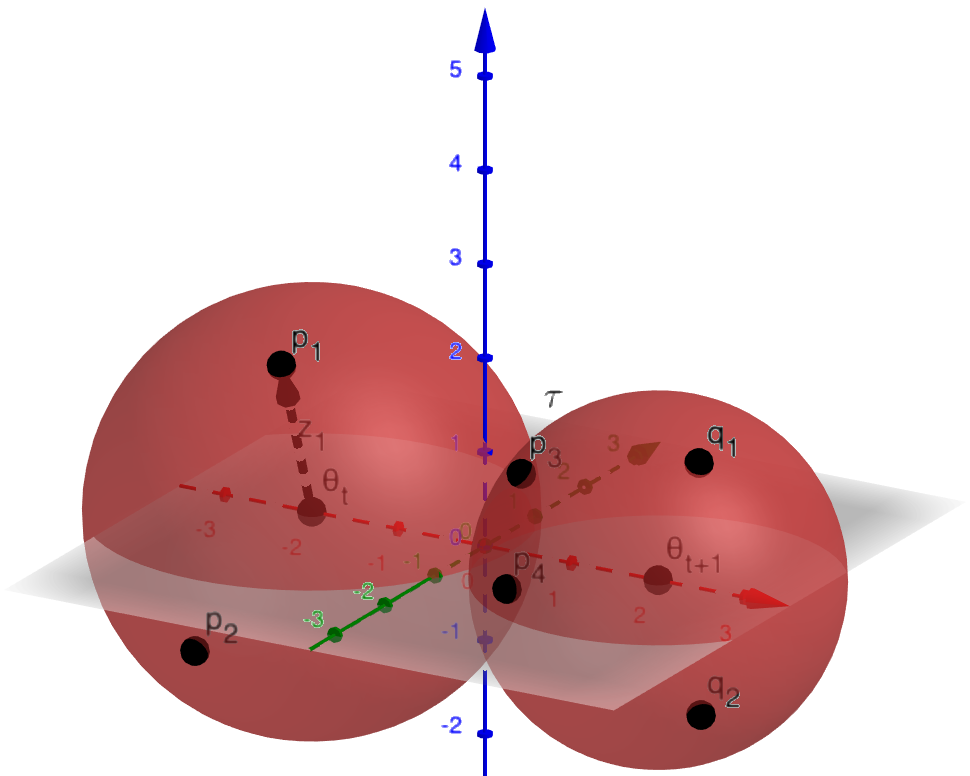}}	
	\hspace{0.3cm}	
	\subfigure[RBO + gradient flows]{\includegraphics[width=0.3\textwidth]{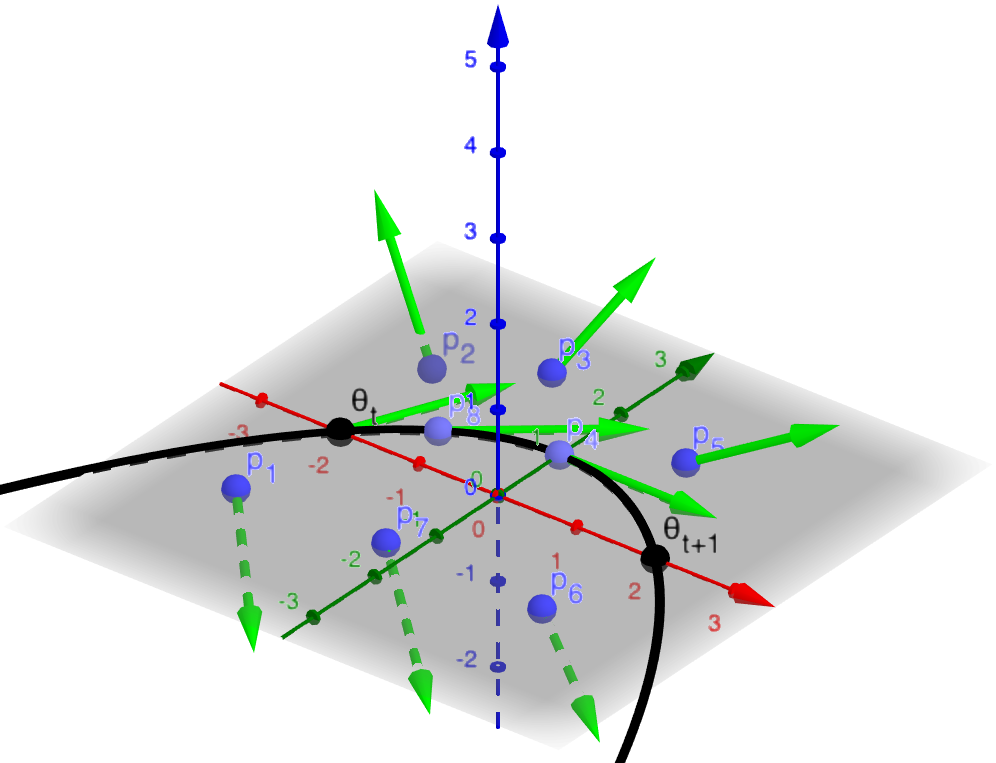}}		
\end{minipage}
	\caption{Comparison of different ES optimization methods: (a) Vanilla ES: to update current point $\theta_{t}$, independent perturbations $\mathbf{p}_{i}$ are chosen. Since perturbations are not reused, all perturbations $\mathbf{q}_{i}$ used in $\theta_{t+1}$ are different from the previous ones. (b) Base RBO: perturbations no longer need to be independent, the $\tau$-percentage of old perturbations closest to the new point $\theta_{t+1}$ are reused. Gradient in $\theta_{t}$ is reconstructed via regression. (c) RBO with gradient flows: gradients are recovered in several point of the neighborhood of $\theta_{t}$ via regression. An approximation of the gradient field in $\mathcal{N}(\theta_{t})$ is computed via matrix-valued kernel interpolation and the update of $\theta_{t}$ is conducted via gradient flow.}
	\label{fig:rbo_scheme} 
\end{figure}

% x1-cropped.png
% x2-cropped.png

The gradient at point $\theta_{t} + \sigma \mathbf{g}_{i}^{t}$ is reconstructed in the analogous way as in Algorithm 1, with the use of the estimator $\widehat{\nabla}_{\mathrm{RBO}}F(\theta_{t}+\sigma \mathbf{g}^{t}_{i})$, where data for the regressor consists of the same $k+1$ points as in Algorithm 1, namely: vector $\theta_{t}$ and vectors $\theta_{t} + \sigma \mathbf{g}_{j}^{t}$ for $j=1,...,k$. The only difference is that now $\theta_{t} + \sigma \mathbf{g}_{i}^{t}$ plays the role of the base vector and other $k$ vectors are interpreted as its perturbed versions. All the reconstructed gradients form a set $\widehat{\mathcal{F}}_{\mathrm{grad}}^{\mathrm{sparse}}$, which can be thought of as a sparse approximation of the gradient field $\mathcal{F}_{\mathrm{grad}}$ of $F$ in $\mathcal{N}(\theta_{t})$.

\subsection{Interpolating gradient field via matrix-valued kernels}

The set of gradients $\widehat{\mathcal{F}}_{\mathrm{grad}}^{\mathrm{sparse}}$ is used to create an interpolation $\widehat{\mathcal{F}}_{\mathrm{grad}}$ of the true gradient field
$\mathcal{F}_{\mathrm{grad}}$ in $\mathcal{N}(\theta_{t})$ via matrix-valued kernels. Below we give a short overview over the theory of matrix-valued kernels \cite{rosasco,pontil,reisert}, which suffices to explain how they can be applied for interpolation.

\begin{definition}[matrix-valued kernels]
A function $K:\mathbb{R}^{d} \times \mathbb{R}^{d} \rightarrow \mathbb{R}^{m} \times \mathbb{R}^{m}$ is a matrix-valued kernel if for every $\mathbf{x},\mathbf{y} \in \mathbb{R}^{d}$ the following holds:
\begin{equation}
K(\mathbf{x},\mathbf{y}) = K(\mathbf{y},\mathbf{x})^{\top}.    
\end{equation}
We call $K$ a \textit{positive definite kernel} if furthermore the following holds. For every set $\mathcal{X} = \{\mathbf{x}_{1},...,\mathbf{x}_{l}\} \subseteq \mathbb{R}^{d}$ the block matrix $K(\mathcal{X},\mathcal{X}) = (K(\mathbf{x}_{i},\mathbf{x}_{j}))_{i,j \in \{1,...,l\}}$ is positive definite.
\end{definition}

As we see, matrix-valued kernels are extensions of their scalar counterparts. As standard scalar-valued kernels can be used to approximate scalar fields via functions from reproducing kernel Hilbert spaces (RKHS) corresponding to these kernels, matrix-valued kernels can be utilized to interpolate vector-valued fields.
The interpolation problem can be formulated as:
\begin{equation}
\mathcal{F}_{*} = \mathrm{argmin} \sum_{j=1}^{m}\frac{1}{N}\sum_{i=1}^{N}(\mathcal{F}_{j}(\mathbf{x}_{i})-\mathbf{y}^{j}_{i})^{2} + \lambda\|\mathcal{F}\|^{2}_{\mathbf{K}}, 
\label{eq:matrixkernel}
\end{equation}
where $\mathcal{F}:\mathbb{R}^{d} \rightarrow \mathbb{R}^{m}$ is a vector-valued function from the RKHS corresponding to $\mathbf{K}$, scalar $\mathbf{y}^{j}_{i}$ is the $j^{th}$ component of the $i^{th}$ vector-valued observations $\mathbf{y}_{i} \in \mathbb{R}^{d}$ for the $i^{th}$ sample $\mathbf{x}_{i}$ from the set $\mathcal{X}$ of $N$ input datapoints and $\|\cdot\|_{\mathbf{K}}$ stands for the norm that the above RKHS is equipped with. The solution to the optimization problem (Equation \ref{eq:matrixkernel}) is given by the formula:
$\mathcal{F}(\mathbf{x}) = \sum_{i=1}^{N} K(\mathbf{x}_{i},\mathbf{x})\mathbf{c}_{i}$, where vectors $\mathbf{c}_{i} \in \mathbb{R}^{d}$ are given by:
\begin{equation}
\label{inverse}
\mathbf{c} = (K(\mathcal{X},\mathcal{X})+\lambda N \mathbf{I}_{Nd \times Nd})^{-1}\mathbf{y}.    
\end{equation} For separable matrix-valued kernels~\cite{sindhwani2013scalable}, such a problem can be solved at a complexity that scales no worse than standard scalar kernel methods. 
Above, $\mathbf{c} \in \mathbb{R}^{Nd}$ is a concatenations of vectors $\mathbf{c}_{1},...,\mathbf{c}_{N}$ and $\mathbf{y}$ is a concatenation of the observations $\mathbf{y}_{1},...,\mathbf{y}_{N}$. 

{\bf Gradient Ascent Flow}: The RBO casts gradient field reconstruction problem as the above interpolation problem, where: $\mathcal{X} = \{\theta_{t}, \theta_{t}+\sigma \mathbf{g}_{1}^{t},...,\theta_{t}+\sigma \mathbf{g}_{k}^{t}\}$ and for $\mathbf{x}_{i} \in \mathcal{X}$ we have: $\mathbf{y}_{i} = \widehat{\nabla}_{\mathrm{RBO}}F(\mathbf{x}_{i})$.
After obtaining the solution $\widehat{\mathcal{F}}_{\mathrm{grad}} = \mathcal{F}_{*}$, the update of the current point $\theta_{t}$ is obtained via the standard gradient ascent flow in the neighborhood of $\theta_{t}$, which is the solution to the following differential equation:
\begin{equation}
d\theta / dt = \widehat{\mathcal{F}}_{\mathrm{grad}}(\theta),
\end{equation}
whose solution can be numerically obtained using such methods as Euler integration.
The comparison of the presented RBO algorithms with baseline ES is schematically presented in Fig. \ref{fig:rbo_scheme}.

\begin{comment}
In this section we want to give a short overview over the theory of matrix-valued kernels
and their application in interpolation. As matrix-valued kernels are an extension of the
well studied scalar-valued kernels, many of the following notions, properties and concepts
are again suitable extensions of their scalar-valued counterparts. For a more extensive
overview with regards to this topic and other approximation schemes involving matrix-valued kernels such as regression, we refer to literature, e.g. [1, 16, 21].
\end{comment}

\subsection{Time Complexity and Distributed Implementation} Our RBO implementation relies on distributed computations, where different workers evaluate $F$ in different subsets of perturbations. The 
time needed to construct all approximate gradients $\mathbf{y}_{i}$ in a given iteration of the algorithm is negligible compared to the time needed for querying $F$ (because calculations of $\mathbf{y}_{i}$ can be also easily parallelized). Computations of $\mathbf{c}$ from Equation \ref{inverse} can be efficiently conducted using separability and random feature maps\cite{rahimi2008random}. We also noted that in practice the gradient-flow extension of the RBO requires many fewer perturbations per iteration than baseline ES and one can use state-of-the-art compact neural networks from \cite{stockholm}, which encode RL policies with a few hundred parameters. This further reduces the cost of computing the gradient field.

\section{Provably Robust Gradient Recovery}
\label{sec:robust_convergence}

It turns out that the RBO with LP decoding ($p=1$, $\alpha=0$) is particularly resilient to noisy measurements. This is surprising at first glance, since we empirically tested (see: Section \ref{sec:exp}) that it is true even when a substantial number of measurements are very inaccurate and when the assumption that noise for each measurement is independent clearly does not hold (e.g. for noisy dynamics or when measurements are spread into simulator calls and real hardware experiments).

In this section we explain why it is the case. We leverage the results from a completely different field: adversarial attacks for database systems \cite{dwork2007price} and explain why RBO with LP decoding can create an accurate sparse approximation $\widehat{\mathcal{F}}_{\mathrm{grad}}$ to the true gradient field $\mathcal{F}_{\mathrm{grad}}$ with loglinear number of measurements per point even if up to $\rho^{*} = 0.2390318914495168...$ of all the measurements are arbitrarily corrupted. Interestingly, we do not require any assumptions regarding sparsity of the gradient vectors. 

We also present convergence results for RBO under certain regularity assumptions regarding functions $F$. These can be translated to 
the results on convergence to local maxima for less regular mappings $F$. All proofs as well as standard definitions of  $L$-Lipschitz, $\lambda$-smooth and $\mu$-strongly concave functions are given in the Appendix. %\textcolor{red}{Aldo: Give commented out here definitions of smoothness, Lipschitz and strong concativity in the Appendix}

\begin{definition}[coefficient $\rho^{*}$]
Let $X \sim \mathcal{N}(0,1)$ and denote: $Y = |X|$. Let $f$ be the $\mathrm{pdf}$ of $Y$ and $F$ be its $\mathrm{cdf}$ function. Define $g(x)=\int_{x}^{\infty} yf(y)dy$. Function $g$ is continuous and decreasing in the interval $[0, \infty]$ and furthermore $g(0) = \mathbb{E}[Y]$. Since $\lim_{x \rightarrow \infty} g(x) = 0$, there exists $x^{*}$ such that $g(x^{*}) = \frac{\mathbb{E}[Y]}{2}$. We define $\rho^{*}$ as:
\begin{equation}
\rho^{*} = 1 - F^{-1}(x*).    
\end{equation}
Its exact numerical value is 
$\rho^{*}= 0.2390318914495168...$.
% 0.239031891449516803895...
\end{definition}

\begin{comment}
\begin{definition}[$\lambda$-smoothness]
A differentiable concave function $F: \Theta \rightarrow \mathbb{R}$ is smooth with parameter $\lambda>0$ if for every pair of points $x,y \in \Theta$:
\begin{equation*}
    \| \nabla F(y) - \nabla F(x)\|_{2} \leq \lambda \| y-x \|_2
\end{equation*}
If $F$ is twice differentiable it is equivalent to $ -\lambda I \preceq \nabla^2 F(x) \preceq 0$ for all $x \in \Theta$.
\end{definition}

\begin{definition}[$L$-Lipschitz]
We say that $F:\Theta \rightarrow \mathbb{R}$ is \emph{Lipschitz} with parameter $L>0$ if for all $x,y \in \Theta$ it satisfies $|F(x) - F(y)| \leq L \| x - y \|_2$.
\end{definition}
\end{comment}

The following result shows the robustness of the RBO gradients under substantial noise:

\begin{lemma}\label{lemma::gradient_approximation}
There exist universal constants $C,c>0$ such that the following holds. Let $F:\Theta \rightarrow \mathbb{R}$ be a $\lambda-$smooth function.
Assume that at most the $\rho^{*}$-fraction of all the measurements are arbitrarily corrupted and the other ones have error at most $\epsilon$. If $\sigma_{t} = \sqrt{ \frac{\epsilon}{d \lambda}}$, $k \geq C d $ and RBO uses LP decoding, with probability $p=1-\exp\left(-cd\right)$ the following holds:
\begin{equation}
\|\nabla_{\mathrm{RBO}}F(\theta_{t}) - \nabla F(\theta)\|_{2} \leq 2C\sqrt{\epsilon d\lambda}.   
\end{equation}
\end{lemma}

\begin{comment}
The above uses the result from database community regarding breaking privacy of database systems with dot-product queries and under adversarial noise \cite{dwork2007price}. The role of the database vector is played by the gradient of the RL blackbox function and noise coming from the measurements corresponds to the adversarial noise added to exact answers of dot-product queries in order to protect the database from being compromised.
\end{comment}

We are ready to state our main theoretical result.

\begin{theorem}\label{theorem::convergence_robust}
Consider a blackbox function $F:\Theta \rightarrow \mathbb{R}$. Assume that $F$ is concave, Lipschitz with parameter $L$ and smooth with smoothness parameter $\lambda$. Assume furthermore that domain $\Theta\subset \mathbb{R}^d$ is convex and has $l_2$ diameter $\mathcal{B} < \infty$. 
Consider $\mathrm{Algorithm}$ \ref{Alg:tr-robust_gradient_recovery}
with $p=1, \alpha=0, \tau = 0, \sigma_t \leq \frac{L}{d\lambda \sqrt{t+1}}, \eta_t = \frac{\mathcal{B}}{L\sqrt{t+1}}$ and the noisy setting in which at each step a fraction of at most $\rho^{*}$ of all measurements $F(\theta_{t}+\sigma_{t}\mathbf{g}_{j}^{t})$ are arbitrarily corrupted for $j=1,2,...,k$. 
There exists a universal constant $c_{1}>0$ such that for any $\gamma \in (0,1)$ and $T \leq \gamma \exp(c_1 d )$, the following holds
with probability at least $1-\gamma$:
\begin{equation*}
    F(\theta^*) -  \left[ \frac{1}{T} \sum_{t=0}^{T-1} F(\theta_t) \right]     \leq\frac{13}{2}\mathcal{B} L \frac{1}{\sqrt{T}},
\end{equation*}

%\note{I think we can get rid of the expectation here.}
\end{theorem}
where $\theta^{*} = \arg \max_{\theta \in \Theta} F(\theta)$. 
If $F$ presents extra curvature properties such as being strongly concave, we can get a linear convergence rate. The following theorem holds:

\begin{comment}
\begin{definition}[Strong concavity]
A function $F:\Theta \rightarrow \mathbb{R}$ is strongly concave with parameter $\mu$ if:
\begin{equation*}
    F(y) \leq F(x) + \langle \nabla F(x), y-x \rangle - \frac{\mu}{2} \| y - x\|^2_2
\end{equation*}
\end{definition}
\end{comment}

\begin{theorem}\label{theorem::convergence_robust_strongconvex}
Assume conditions from Theorem \ref{theorem::convergence_robust} and furthermore that $F$ is strongly concave with parameter $\mu$. 
Take $\mathrm{Algorithm}$ \ref{Alg:tr-robust_gradient_recovery}
with $p=1, \alpha = 0, \tau = 0,  \sigma_t \leq \frac{L^2}{d \mathcal{B} \mu \lambda (t+1)}, \eta_t = \frac{1}{\mu(t+1)}$ acting 
in the noisy environment in which at each step a fraction of at most $\rho^{*}$ of all measurements $F(\theta_{t}+\sigma_{t}\mathbf{g}_{j}^{t})$ 
are arbitrarily corrupted for $j=1,2,...,k$. There exists a universal constant $c_{1}>0$ such that for any $\gamma \in (0,1)$ and $T \leq \gamma \exp(c_1 d )$, 
with probability at least $1-\gamma$:
\begin{equation*}
F(\theta^*) - \left[ \frac{1}{T}\sum_{t=0}^{T-1}  F(\theta_t) \right] \leq \frac{6 L^2}{\mu}\frac{(1+\log(T))}{T}.
\end{equation*}

%\note{I think we can get rid of the expectation here.}
\end{theorem}
To summarize, even if up to $23\%$ of all the measurements are arbitrarily corrupted, RBO can provably recover gradients to high accuracy! We provide empirical evidence of this result next.
\section{Empirical Analysis}
\label{sec:exp}

\vspace{-3mm}
\begin{figure}[H]
\begin{minipage}{1.0\textwidth}
	\subfigure{\includegraphics[keepaspectratio, width=0.255\textwidth]{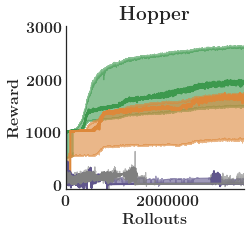}}
	\subfigure{\includegraphics[keepaspectratio, width=0.235\textwidth]{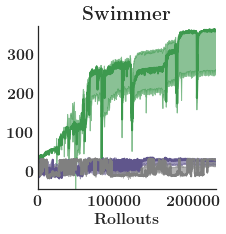}}
	\subfigure{\includegraphics[keepaspectratio, width=0.235\textwidth]{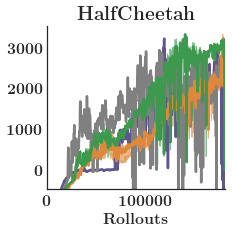}}
	\subfigure{\includegraphics[keepaspectratio, width=0.24\textwidth]{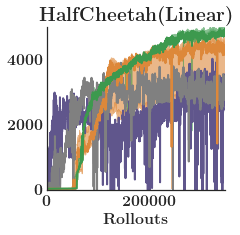}}
\end{minipage}
\begin{minipage}{1.0\textwidth}
	\subfigure{\includegraphics[keepaspectratio, width=0.26\textwidth]{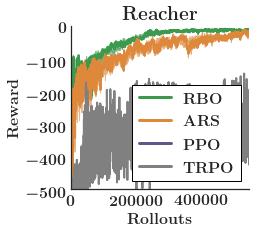}}
	\subfigure{\includegraphics[keepaspectratio, width=0.24\textwidth]{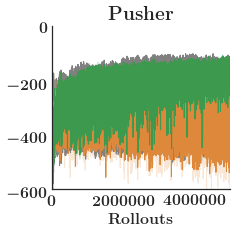}}
	\subfigure{\includegraphics[keepaspectratio, width=0.235\textwidth]{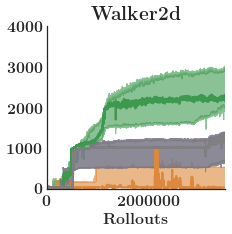}}
	\subfigure{\includegraphics[keepaspectratio, width=0.24\textwidth]{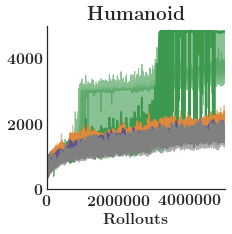}}
\end{minipage}
    \caption{\small{Performance of RBO, $\mathrm{ARS}$, PPO and TRPO on seven benchmark RL environments. Results presented are the median from $3$ seeds, with the $\mathrm{min}$ and $\mathrm{max}$ shaded. In each case, $20\%$ of the rewards are significantly corrupted. In some cases, this noise led to drastically worse performance for non-RBO methods, which we omit from the plots, but include in the tabular data in Table 1.}}
	\label{fig:mujoco} 
\vspace{-5mm}	
\end{figure}

The real world is often far noisier than environments typically used for benchmarking RL algorithms. With this in mind, the primary goal of our experiments is to demonstrate that RBO is able to efficiently learn good policies in the presence of noise, where other approaches fail. To investigate this, we consider two settings: 
\begin{enumerate}
    \item $\mathrm{OpenAI}$ $\mathrm{Gym}$ \cite{Gym} 
MuJoCo environments, where $20\%$ of the measurements are significantly corrupted (we show that adding noise presents a challenge to baseline algorithms). 
    \item Real-world quadruped locomotion tasks, where sim-to-real transfer is non-trivial.
\end{enumerate}

We run RBO with LP-decoding to obtain provably noise-robust reconstruction of ES gradients.

\vspace{-2mm}
\paragraph{OpenAI Gym:}
We conducted an exhaustive analysis of the proposed class of RBO algorithms on the following $\mathrm{OpenAI}$ $\mathrm{Gym}$ \cite{Gym} 
benchmark RL tasks: $\mathrm{Swimmer}$, $\mathrm{HalfCheetah}$, $\mathrm{Hopper}$, $\mathrm{Walker2d}$, $\mathrm{Humanoid}$,  $\mathrm{Pusher}$ and $\mathrm{Reacher}$. 
All but $\mathrm{HalfCheetah}$ $\mathrm{Linear}$ experiments are for a policy encoded by feedforward neural networks with two hidden layers of size $h=41$ each and $\mathrm{tanh}$ nonlinearities. $\mathrm{HalfCheetah}$ $\mathrm{Linear}$ is for the linear architecture.
We compare RBO to state-of-the-art ES algorithm $\mathrm{ARS}$ \cite{horia}, as well as two state-of-the art policy gradient algorithms: TRPO \cite{schulman2015trust} and PPO \cite{schulman2017proximal}. In all cases, we corrupt $20\%$ of the measurements. 
As we show in Fig. \ref{fig:mujoco}, the noise often renders the other algorithms unable to learn optimal policies, yet RBO remains unscathed and consistently learns good policies for all tasks. Under the presence of substantial noise the other methods often drastically underperform RBO, as we show in Table 1.

\vspace{-3.0mm} 
\begin{table}[H]
\small
  \label{rl_results}
  \centering
   \begin{tabular}{l*6{c}}
        \toprule
        \multicolumn{6}{c}{\textbf{Median reward after \# rollouts}}                   \\
        \cmidrule(r){3-6}
        \textbf{Environment}     & \textbf{Rollouts}     & RBO &  $\mathrm{ARS}$ & TRPO & PPO \\
        \midrule
        $\mathrm{HalfCheetah}$ $\mathrm{(Linear)}$ & $2.10^5$  & \textbf{4220} & 4205 & \textcolor{red}{\textbf{\textbf{-Inf}}} & \textcolor{red}{\textbf{-Inf}}   \\
        $\mathrm{HalfCheetah}$ $\mathrm{(Toeplitz)}$ & $2.10^5$  & \textbf{3299} & 3163 & \textcolor{red}{\textbf{-Inf}} & \textcolor{red}{\textbf{-Inf}}   \\        
        $\mathrm{Swimmer}$ & $2.10^5$  & \textbf{360} & \textcolor{red}{\textbf{-Inf}} & 32 & 30   \\
        $\mathrm{Walker2d}$ & $2.10^6$  & \textbf{2230} & 172 & 996 & 312   \\
        $\mathrm{Hopper}$ & $1.10^6$  & \textbf{1503} & 1408 & 427 & 256   \\
        $\mathrm{Humanoid}$ & $5.10^6$  & \textbf{4865} & 2355 & 2028 & 2129   \\
        $\mathrm{Pusher}$ & $1.10^6$  & \textbf{-155} & -199 & \textcolor{red}{\textbf{-Inf}} & \textcolor{red}{\textbf{-Inf}}   \\
        $\mathrm{Reacher}$ & $5.10^5$  & \textbf{-7} & -19 & \textcolor{red}{\textbf{-Inf}} & \textcolor{red}{\textbf{-Inf}}  \\
        \bottomrule
      \end{tabular}
  \caption{\small{Median rewards obtained across $k=5$ seeds for seven RL environments. Bold represents the best performing algorithm in each environment, red indicates failure to learn.}}     
\vspace{-4mm}  
\end{table}
%\vspace{-4mm}

\vspace{-4mm}
\paragraph{Quadruped Locomotion:} We tested RBO on quadruped locomotion tasks for a Minitaur robot~\cite{kenneally2016design} (see: Fig. \ref{fig:minitaur}),
with different reward functions for different locomotion tasks.
Minitaur has $4$ legs and $8$ degrees of freedom, where each leg has the ability to swing and extend to a certain degree using the PD controller 
provided with the robot. We train our policies in simulation using the pybullet environment modeled after the robot \cite{coumans}.
To learn walking for quadrupeds, we use architectures called \textit{Policies Modulating Trajectory Generators} (PMTGs) that have been recently proposed 
in \cite{atil}. They incorporate basic cyclic characteristics of the locomotion and leg movement primitives by using trajectory 
generators, a parameterized function that provides cyclic leg positions. The policy is responsible for modulating and adjusting leg trajectories. The results fully support our previous findings.
While without noise RBO and ARS (used as state-of-the-art method for these tasks \cite{atil}) perform similarly, in the presence of noise RBO is superior to ARS.

\vspace{-5mm}
\begin{figure}[H]
\centering
\begin{minipage}{0.99\textwidth}
	\raisebox{0.22\height}{\subfigure{\includegraphics[keepaspectratio, width=0.225\textwidth]{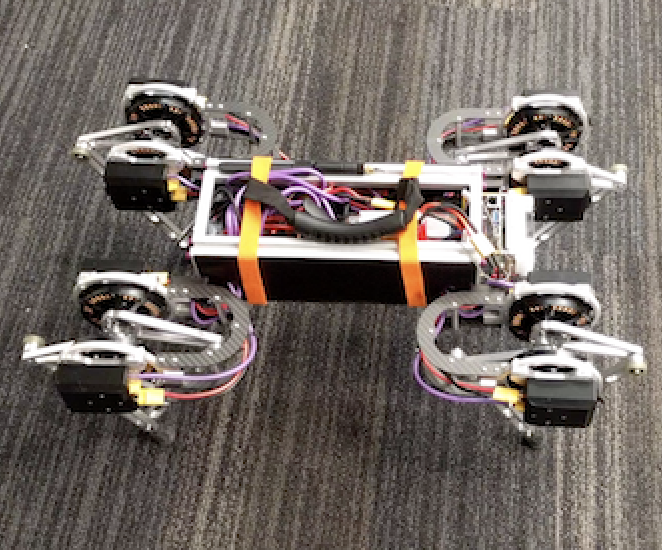}}}
	\subfigure{\includegraphics[keepaspectratio, width=0.26\textwidth]{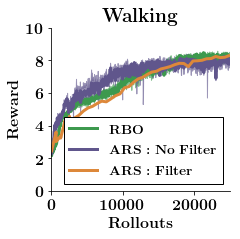}}	\subfigure{\includegraphics[keepaspectratio, width=0.23\textwidth]{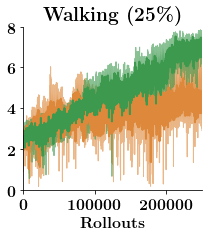}}
	\subfigure{\includegraphics[keepaspectratio, width=0.245\textwidth]{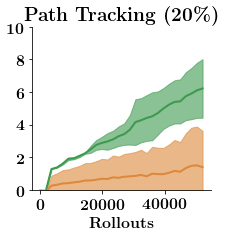}}
\end{minipage}
\vspace{-5mm}	
    \caption{\small{Left: The Minitaur robot. Right: Performance of RBO and $\mathrm{ARS}$ for two Minitaur simulated environments: forward walking (without noise and with $25\%$ noisy measurements) and path tracking with $20\%$ noisy measurements.}}
	\label{fig:minitaur} 
\vspace{-5mm}	
\end{figure}

\section{Conclusion}
\vspace{-3mm}
We proposed a new class of ES algorithms, called RBO, for optimizing RL policies that utilize gradient flows induced by vector field interpolation via matrix-valued kernels. The interpolators rely on general regularized regression methods that provide sample complexity reduction through sample reuse. We show empirically and theoretically that RBO is much less sensitive to noisy measurements, which are notoriously ubiquitous in robotics applications, than state-of-the-art baseline algorithms.

% no \bibliographystyle is required, since the corl style is automatically used.
% \bibliography{example}  % .bib
\bibliographystyle{abbrv}
\bibliography{rbo}

\newpage
 \section{Appendix}

\subsection{Definitions}
Here we introduce the definitions of $\lambda$-smoothenss and $L$-lipshitz. These are standard definitions in the optimization literature which we reproduce here for clarity (see for example \cite{hazan2016introduction}). 

\begin{definition}($\lambda$-smoothness):  A differentiable concave function $F:\Theta \rightarrow \mathbb{R}$ is smooth with parameter $\lambda > 0$ if for every pair of points $x,y \in \Theta$:
\begin{equation*}
    \|  \nabla F(y) - \nabla F(x) \|_2 \leq \lambda \| y-x \|_2
\end{equation*}
If $F$ is twice differentiable it is equivalent to $-\lambda I \preceq \nabla^2 F(x) \preceq 0$ for all $x \in \Theta$.
\end{definition}

\begin{definition}($L$-Lipschitz): We say that $F:\theta \rightarrow \mathbb{R}$ is Lipschitz with parameter $L > 0$ if for all $x,y \in \Theta$ it satisfies $| F(x) - F(y) | \leq L \| x-y \|_2$.

\end{definition}

\begin{definition}($\mu$-Strong Concavity): A function $F: \Theta \rightarrow \mathbb{R}$ is strongly concave with parameter $\mu$ if:
\begin{equation*}
    F(y) \leq F(x) + \langle \nabla F(x), y-x \rangle - \frac{\mu}{2} \| y-x \|^2_2
\end{equation*}

\end{definition}

\subsection{Proof of Lemma \ref{lemma::gradient_approximation}}

In this section we prove Lemma \ref{lemma::gradient_approximation} which we reproduce below for readability. This result is concerned with the case when a constant proportion of the measurements are arbitrarily corrupted, while the remaining ones are corrupted by a small amount $\epsilon$. We quantify the degree of corruption experienced by our gradient estimator. Let $\nabla F(\theta_t)$ be the real gradient of $F$ at $\theta_t$. Before proving the main Lemma of this section, let's show that whenever $\sigma_t$ is very small, and the evaluation of $F$ is noiseless, a difference of function evaluations is close to a dot product between the gradient of $F$ and the displacement vector. 
 
\begin{lemma}
    $F(\theta_t + \sigma_t v^{(t)}_j) - F(\theta_t) =  \langle \nabla F(\theta_t), \sigma_t \mathbf{g}_j^{(t)} \rangle + \xi_t$ with $|\xi_t| \leq \sigma_t^2 d \lambda$. 
\end{lemma}

This follows immediately from a Taylor expansion and the smoothness assumption on $F$. %We we state and prove a more general version of Lemma \ref{lemma::gradient_approximation}.

\begin{lemma}\label{lemma::closeness_gradient_with_errors}
There exist universal constants $C,c>0$ such that the following holds. Let $F:\Theta \rightarrow \mathbb{R}$ be a $\lambda-$smooth function.
Assume that at most the $\rho^{*}$-fraction of all the measurements are arbitrarily corrupted and the other ones have error at most $\epsilon$. If $\sigma_t = \sqrt{ \frac{\epsilon}{d \lambda}}$, $k \geq C d $ and RBO uses LP decoding, the following holds with probability $p=1-\exp\left(-cd\right)$:
\begin{equation}\label{equation::bias_gradient_approximation}
\|\nabla_{\mathrm{RBO}}F(\theta_{t}) - \nabla F(\theta)\|_{2} \leq 2C\sqrt{\epsilon d\lambda}.   
\end{equation}
\end{lemma}
\begin{proof}
We use $F(\theta_t + \sigma_t \mathbf{g}_j^{(t)})$ as proxy measurements for $\langle \nabla F(\theta_t), \sigma_t \mathbf{g}_j^{(t)} \rangle$. Since $F(\theta_t + \sigma_t v^{(t)}_j) - F(\theta_t) =  \langle \nabla F(\theta_t), \sigma_t \mathbf{g}_j^{(t)} \rangle + \xi_t$ with $|\xi_t| \leq \sigma_t^2 d \lambda$, and we assume the measurements of $F(\theta_t + \sigma_t \mathbf{g}_j^{(t)} ) - F(\theta)$ are either completely corrupted or corrupted by a noise of magnitude at most $\epsilon$, the total displacement for the normalized dot product component of the objective equals $\sigma_t d \lambda + \frac{\epsilon}{\sigma_t}$. Setting $\sigma_t = \sqrt{ \frac{\epsilon}{d \lambda}}$ means the total error can be driven down to $2\sqrt{\epsilon d \lambda}$ by this choice of $\sigma_t$. After these observations, a direct application of Theorem 1 in \cite{dwork2007price} yields the result.
\end{proof}

Notably, the error in Equation \ref{equation::bias_gradient_approximation} cannot be driven to zero unless the errors of magnitude $\epsilon$ are driven to zero themselves. This is in direct contrast with the supporting lemma we prove in the next section as a stepping stone towards proving Thoerem \ref{theorem::convergence_robust}. Nonsurprisingly our result has dependence on the irreducible error $\epsilon$. There is no way to get around the dependence on this error.   %This result implies there is a limit in gradien recovery guarantees of our

\subsection{Proof of Theorem \ref{theorem::convergence_robust}}

%\textcolor{red}{Change all $v$ to $\mathbf{g}$}.

We start with a result that is similar in spirit to Lemma \ref{lemma::gradient_approximation}. The assumptions behind Theorem \ref{theorem::convergence_robust} differ from those underlying Lemma \ref{lemma::gradient_approximation} in that we only assume the presence of a constant fraction of arbitrary perturbations on the measurements. All the remaining measurements are assumed to be exact. We show the recovered gradient $\hat{\nabla}_{RBO}$ is close to the true gradient. This distance is controlled by the smoothing parameters $\{   \sigma_t \}$. 

\begin{lemma} \label{lemma::closeness_gradient}
There exist universal constants $c_1, c_2$ such that if for any $t$ if up to $\rho^*$ fraction of the entries of $y_t$ are arbitrarily corrupted, the gradient recovery optimization problem with input $\theta_t$ satisfies:
\begin{equation}
    \| \hat{\nabla}_{RBO} F(\theta_t) - \nabla F(\theta_t) \| \leq \sigma_t d \lambda
\end{equation}
Whenever $k \geq  c_1 d$ and with probability $1-\exp\left( -c_2 d  \right)$
\end{lemma}

The proof of Lemma \ref{lemma::closeness_gradient} and the constants $c_1, c_2$ are a result of a direct application of Theorem 1 in \cite{dwork2007price}.

Assume from now on the domain $\Theta\subset \mathbb{R}^d$ is convex and has $l_2$ diameter $\mathcal{B} < \infty$. We can now show the first order Taylor approximation of $F$ around $\theta_t$ that uses the true gradient and the one using the RBO gradient are uniformly close:

\begin{lemma}\label{lemma::sup_gradient}
The following bound holds:
For all $\theta_t \in \Theta$:
\begin{equation*}
    \sup_{\theta \in \Theta} | \langle \theta - \theta_t, \hat{\nabla}_{RBO} F(\theta_t)  \rangle - \langle \theta - \theta_t, \nabla F(\theta_t) \rangle | \leq  \mathcal{B} \sigma_t d \lambda
\end{equation*}
\end{lemma}

The next lemma provides us with the first step in our convergence bound: 

\begin{lemma}\label{lemma::fundamental_convergence_inequality}
For any $\theta^*$ in $\Theta$, it holds that:
\begin{align*}%\label{eq:fundamental_convergence_lemma}
    2\left( F(\theta^*) - F(\theta_t)   \right) &\leq \frac{ \| \theta_t - \theta^*\|^2_2  - \|\theta_{t+1} - \theta^* \|_2^2 }{\eta_t} \\
    & \quad + \eta_t \left(  \| \nabla F(\theta_t) \| + \sigma_t d \lambda\right)^2 \\
    & \quad + 2\mathcal{B}\sigma_t d \lambda 
\end{align*}

\end{lemma}

\begin{proof}
%By concavity of $F$ it is true that $\langle \nabla F(\theta_t), \theta^* - \theta_t \rangle \geq F(\theta*) - F(\theta_t)$. 

Recall that $\theta_{t+1}$ is the projection of $u_{t+1}$ to a convex set $\Theta$. And that $u_{t+1} = \theta_t + \eta_t \hat{\nabla}_{RBO} F(\theta_t)$. As a consequence:

\begin{align}
\begin{split}\label{eq::potential_gradient}
    \| \theta_{t+1} - \theta^* \|^2 &\leq  \| \theta_t + \eta_t \hat{\nabla}_{RBO}F(\theta_t) - \theta^* \|^2 \\
    & = \| \theta_t - \theta^* \|^2 + \eta_t^2 \| \hat{\nabla}_{RBO}F(\theta_t) \|^2 \\
    &\quad- 2\eta_t \langle \hat{\nabla}_{RBO}F(\theta_t), \theta^* - \theta_t \rangle 
\end{split}
\end{align}

Lemma \ref{lemma::closeness_gradient} and the triangle inequality imply:
\begin{equation*}
\|\hat{\nabla}_{RBO} F(\theta_t)  \|^2 \leq \left( \| \nabla F(\theta_t) \| + \sigma_t d \lambda   \right)^2
\end{equation*}

This observation plus Lemma \ref{lemma::sup_gradient} applied to Equation \ref{eq::potential_gradient} implies:
\begin{align*}
    2\langle \nabla F(\theta_t), \theta^* - \theta_t \rangle  &\leq \frac{ \| \theta_t - \theta^*\|^2_2  - \|\theta_{t+1} - \theta^* \|_2^2 }{\eta_t} \\
    & \quad + \eta_t \left(  \| \nabla F(\theta_t) \| + \sigma_t d \lambda\right)^2 \\
    & \quad + 2\mathcal{B}\sigma_t d \lambda 
\end{align*}

Since concavity of $F$ implies $F(\theta^* ) - F(\theta_t) \leq  \langle \nabla F(\theta_t), \theta^* - \theta_t \rangle$, the result follows.
\end{proof}

We proceed with the proof of Theorem \ref{theorem::convergence_robust}:

\begin{align}
\begin{split}
    2 \sum_{t=0}^{T-1} \left( F(\theta^*) -  F(\theta_t)   \right) \leq \sum_{t=0}^{T-1} \frac{ \| \theta_t - \theta^*\|^2_2  - \|\theta_{t+1} - \theta^* \|_2^2 }{\eta_t} \\
    \quad + \sum_{t=0}^{T-1} \eta_t \left( \| \nabla F(\theta_t)  \| + \sigma_t d \lambda \right)^2  
    \quad + 2 \mathcal{B}\sigma_t d \lambda \\
    \leq \sum_{t=0}^{T-1} \|\theta_{t} - \theta* \|^2 \left( \frac{1}{\eta_t} - \frac{1}{\eta_{t-1}} \right) \\
    \quad + \sum_{t=0}^{T-1} \eta_t (L+\sigma_t d \lambda)^2  + 2 \mathcal{B}\sigma_t d \lambda, 
\end{split}    
\end{align}

where we set $\frac{1}{\eta_{-1}} = 0$. The first inequality is a direct consequence of Lemma \ref{lemma::fundamental_convergence_inequality}. The second inequality follows because $\|  \theta_{T} - \theta^* \| \geq 0$ and $\| \nabla F(\theta) \| \leq L$ for all $\theta \in \Theta$.

As long as, $\sigma_t \leq \frac{L}{d \lambda \sqrt{t+1}}$ and $\eta_t = \frac{\mathcal{B}}{L \sqrt{t+1}}$ we have:

\begin{equation*}
\sum_{t=0}^{T-1} F(\theta^*) -F(\theta_t) \leq \frac{13}{2}\mathcal{B} L \sqrt{T} %\frac{\mathcal{B} L \sqrt{T}}{2} + 4\mathcal{B} L \sqrt{T} +  2\mathcal{B} L \sqrt{T} 
\end{equation*}

Since $\sum_{t=1}^{T} \frac{1}{\sqrt{t}} \leq 2\sqrt{T}$, Theorem \ref{theorem::convergence_robust} follows.

% Notice that $\mathcal{B} \leq 2 \mathcal{R}$ \note{get rid of one of this two bounds}. If $\sigma_t \leq \frac{L}{\lambda d \sqrt{T}}$ and $\eta = \frac{\mathcal{R}}{2L\sqrt{T}}$ we get the following bound:
% \begin{equation}
%     \sum_{t=0}^{T-1}  \langle \nabla F(\theta_t) ,  \theta^* - \theta_t \rangle \leq 2\mathcal{R} L \sqrt{T} + 4\mathcal{R} L \sqrt{T}= 6 \mathcal{R} L \sqrt{T}.
% \end{equation}

% By concavity of $F$ it is true that $\langle \nabla F(\theta_t) , \theta^* - \theta_t  \rangle \geq  F(\theta^*) - F(\theta_t)$. We conclude that: 
% \begin{equation}
%     \sum_{t=0}^{T-1} \left(F(\theta^*)  -  F(\theta_t) \right) \leq 2\mathcal{R} L \sqrt{T} + 4\mathcal{R} L \sqrt{T}= 6 \mathcal{R} L \sqrt{T}.
% \end{equation}

% Since $\sigma_t d \lambda \leq L$ and therefore $L+ \sigma_t d \lambda \leq 2L$. The results of Theorem \ref{theorem::convergence_robust} follow.

%%%%%%%%%%%%%%%%%%%%%%%%%%%%%%%%%%
\subsection{Proof of Theorem \ref{theorem::convergence_robust_strongconvex}}

In this section we flesh out the convergence results for robust gradient descent when $F$ is assumed to be Lipschitz with parameter $L$, smooth with parameter $\lambda$ and strongly concave with parameter $\mu$.

\begin{lemma}\label{lemma::fundamental_convergence_inequality_strong}
For any $\theta^*$ in $\Theta$, it holds that:
\begin{align*}%\label{eq:fundamental_convergence_lemma}
    2\left( F(\theta^*) -  F(\theta_t)   \right) &\leq \frac{ \| \theta_t - \theta^*\|^2_2  - \|\theta_{t+1} - \theta^* \|_2^2 }{\eta_t} -\mu \| \theta_t - \theta^*\|^2 + \\
    & \eta_t \left(  \| \nabla F(\theta_t) \| + \sigma_t d \lambda\right)^2  + 2\mathcal{B}\sigma_t d \lambda 
\end{align*}

\end{lemma}

\begin{proof}
%By concavity of $F$ it is true that $\langle \nabla F(\theta_t), \theta^* - \theta_t \rangle \geq F(\theta*) - F(\theta_t)$. 

Recall that $\theta_{t+1}$ is the projection of $u_{t+1}$ to a convex set $\Theta$. And that $u_{t+1} = \theta_t + \eta_t \hat{\nabla}_{RBO} F(\theta_t)$. As a consequence:

\begin{align}
\begin{split}\label{eq::potential_gradient}
    \| \theta_{t+1} - \theta^* \|^2 &\leq  \| \theta_t + \eta_t \hat{\nabla}_{RBO}F(\theta_t) - \theta^* \|^2 \\
    & = \| \theta_t - \theta^* \|^2 + \eta_t^2 \| \hat{\nabla}_{RBO}F(\theta_t) \|^2 \\
    &\quad- 2\eta_t \langle \hat{\nabla}_{RBO}F(\theta_t), \theta^* - \theta_t \rangle 
\end{split}
\end{align}

Lemma \ref{lemma::closeness_gradient} and the triangle inequality imply:
\begin{equation*}
\|\hat{\nabla}_{RBO} F(\theta_t)  \|^2 \leq \left( \| \nabla F(\theta_t) \| + \sigma_t d \lambda   \right)^2
\end{equation*}

This observation plus Lemma \ref{lemma::sup_gradient} applied to Equation \ref{eq::potential_gradient} implies:
\begin{align*}
    2\langle \nabla F(\theta_t), \theta^* - \theta_t \rangle  &\leq \frac{ \| \theta_t - \theta^*\|^2_2  - \|\theta_{t+1} - \theta^* \|_2^2 }{\eta_t} \\
    & \quad + \eta_t \left(  \| \nabla F(\theta_t) \| + \sigma_t d \lambda\right)^2 \\
    & \quad + 2\mathcal{B}\sigma_t d \lambda 
\end{align*}

Since strong concavity of $F$ implies $F(\theta^* ) - F(\theta_t) \leq  \langle \nabla F(\theta_t), \theta^* - \theta_t \rangle - \frac{\mu}{2} \|\theta_t - \theta^* \|^2$ the result follows.
\end{proof}

The proof of Theorem \ref{theorem::convergence_robust_strongconvex} follows from the combination of the last few lemmas. Indeed, we have:

\begin{align*}
    2 \sum_{t=0}^{T-1} \left( F(\theta^*) -  F(\theta_t)   \right) &\leq \sum_{t=0}^{T-1} \frac{ \| \theta_t - \theta^*\|^2_2  - \|\theta_{t+1} - \theta^* \|_2^2 }{\eta_t} -\\
    &\mu \| \theta_t - \theta^*\|^2 +\\
    & \sum_{t=0}^{T-1} \eta_t \left( \| \nabla F(\theta_t)  \| + \sigma_t d \lambda \right)^2  + 2 \mathcal{B}\sigma_t d \lambda \\
    &\leq \underbrace{\sum_{t=0}^{T-1} \|\theta_{t} - \theta* \|^2 \left( \frac{1}{\eta_t} - \frac{1}{\eta_{t-1}} - \mu \right) }_{I} + \\
    & \sum_{t=0}^{T-1} \eta_t (L+\sigma_t d \lambda)^2  + 2 \mathcal{B}\sigma_t d \lambda, 
\end{align*}

where we set $\frac{1}{\eta_{-1}} = 0$. the first inequality is a direct consequence of Lemma \ref{lemma::fundamental_convergence_inequality_strong}. The second inequality follows because $\|  \theta_{T} - \theta^* \| \geq 0$ and $\| \nabla F(\theta) \| \leq L$ for all $\theta \in \Theta$. Since $\eta_t = \frac{1}{\mu*(t+1)}$, $\frac{1}{\eta_t} - \frac{1}{\eta_{t-1}} = \mu$ for all $t= 0, \cdots , T-1$ the term labeled I in the inequality above vanishes.

As long as $\sigma_t \leq \frac{L^2}{d \mathcal{B} \mu \lambda(t+1)}$, we have:

\begin{equation*}
\sum_{t=0}^{T-1} F(\theta^*) -F(\theta_t) \leq \frac{6 L^2}{\mu}(1+\log(T)) 
\end{equation*}

Since $\sum_{t=1}^{T} \frac{1}{t} \leq 1 + \log(T)$, Theorem \ref{theorem::convergence_robust_strongconvex} follows.

\subsection{Further experimental details}

In all experiments we used learning rate $\eta = 0.01$. ES algorithms (RBO and ARS) were applying smoothing parameter $\sigma = 0.1$. Furthermore, ARS used both state and reward renormalization, as described in \cite{horia}. In RBO experiments with gradient flows we used matrix-valued kernels based on the class of radial basis function scalar kernels (RBFs).
Measurement noise was added by corrupting certain percentages of the computed rewards at each iteration of the algorithm. 

We used implementation of the Trust Region Policy Optimization ($\mathrm{TRPO}$) \cite{schulman2015trust} algorithm from \cite{baselines}.
We applied default hyper-parameters. Similarly, we used Proximal Policy Optimization ($\mathrm{PPO}$) \cite{schulman2017proximal}
implementation from \cite{baselines} and applied default hyper-parameters.

\end{document}